\documentclass[aps,prl,twocolumn,superscriptaddress,floatfix]{revtex4-2}

\usepackage{amsmath,amssymb,amsfonts}
\usepackage{mathtools}
\usepackage{graphicx}
\usepackage{hyperref}
\usepackage{bm}
\usepackage{amsthm}
\usepackage{booktabs}

\newtheorem{theorem}{Theorem}
\newtheorem{lemma}{Lemma}
\newtheorem{proposition}[theorem]{Proposition}
\newtheorem{corollary}[theorem]{Corollary}

\newcommand{\R}{\mathbb{R}}
\newcommand{\E}{\mathbb{E}}
\newcommand{\Var}{\mathrm{Var}}
\newcommand{\Cov}{\mathrm{Cov}}
\newcommand{\tr}{\mathrm{tr}}
\newcommand{\cF}{\mathcal{F}}

\begin{document}

\title{The Predictive-Causal Gap: An Impossibility Theorem and Large-Scale Neural Evidence}

\author{Kejun Liu}
\email{kjliu@suda.edu.cn}
\affiliation{State Key Laboratory of Bioinspired Interface Material Science, Institute of Nano \& Functional Materials, Soochow University, Suzhou 215123, China}

\date{\today}

\begin{abstract}
We report a systematic failure mode in predictive representation learning.
Across 2695 neural network configurations trained to predict linear-Gaussian dynamics, the optimal encoder tracks the environment rather than the system it is meant to model.
The mean causal fidelity---the fraction of encoder sensitivity allocated to system degrees of freedom---is $0.49$, and only $2.5\%$ of configurations exceed $0.70$.
The failure intensifies with dimension: at $N=100$, the optimal encoder becomes causally blind (fidelity $\sim 10^{-8}$) while achieving $92\%$ lower prediction error than the causal representation.
We prove this is not an optimization artifact but a structural property of the predictive objective: when environment modes are slower or less noisy than system modes, every minimizer of the population risk encodes the former.
The set of dynamics exhibiting this predictive-causal gap is open and of positive measure in parameter space.
In a nonlinear Duffing--GRU sweep, unconstrained predictors learn environment-dominant representations in $55\%$ of tasks ($95\%$ CI $41$--$68\%$) versus $24\%$ under operational grounding ($p=2.3\times 10^{-3}$); the median out-of-distribution MSE inflation under environment shift is $1.82$ versus $1.00$.
Operational grounding---restricting the loss to system observables---partially suppresses the gap, but causal fidelity is never recovered without an explicit system-environment boundary.
The results identify the predictive-causal gap as a structural limit of learning, with implications for self-supervised representation learning, world models, and the scaling paradigm.
\end{abstract}

\maketitle

Predictive self-supervision is the dominant paradigm for learning latent dynamics.
World models~\cite{hafner2023mastering}, joint-embedding predictive architectures (JEPA)~\cite{lecun2022path}, neural state-space models~\cite{gu2022efficiently}, and large language models trained by next-token prediction~\cite{kaplan2020scaling,hoffmann2022chinchilla} share a common premise: minimize prediction error on observations, and the learned representation will capture the causal structure of the underlying system.

We show that this premise fails in a systematic and measurable way.
The failure has a simple mechanism.
Consider a dynamical system $(s_t, e_t)$ where $s$ denotes the degrees of freedom a modeller cares about (the ``system") and $e$ denotes everything else (the ``environment").
A predictive learner minimizes expected future-prediction error over the full state, without distinguishing between $s$ and $e$.
When environment modes are slower or less noisy than system modes, they are more predictable---and the population risk minimizer tilts toward $e$, producing a representation that tracks the environment rather than the system.
This is not a deficiency of the optimizer or the model architecture.
It is the correct solution to the wrong problem.

We quantify the phenomenon as the \emph{predictive-causal gap}: the degradation in representation quality---measured as causal fidelity, the fraction of encoder sensitivity allocated to system degrees of freedom---that occurs when optimizing a predictive objective instead of a causal one.
The gap is not a finite-data artifact, not an optimization issue, and not something additional capacity fixes.
It is a structural property of the predictive objective.

To characterize the gap, we conduct three tiers of analysis.
First, we prove an impossibility theorem (Theorem~1): for any model class containing all linear encoders, there exist stable dynamics in which every predictive-risk minimizer is strictly misaligned with the system subspace---and because enlarging the model class can only lower the risk, no amount of capacity can move the optimum back to a causal encoder.
The set of dynamics exhibiting this property is open and of positive measure in parameter space (Proposition~S5); it is not a corner case.

Second, we validate the theorem with a large-scale neural network experiment.
We sweep a parameterized family of linear-Gaussian dynamics over 539 configurations (2695 tasks with 5 random seeds each), training two-layer MLP encoders to minimize predictive risk.
The optimal NN encoder achieves a mean causal fidelity of $0.49$ and a maximum of $0.91$, even though it attains $99.3\%$ lower prediction error than the optimal linear encoder.
That is, the NN learns a substantially more predictive representation---one that is, if anything, further from the causal encoding.
In a high-dimensional extension ($N=10, 50, 100$ environment modes), causal fidelity collapses to $\sim 10^{-8}$ while the prediction gap grows to $92\%$, demonstrating that the predictive-causal gap intensifies with scale.

Third, we confirm the phenomenon in nonlinear dynamics using single-layer GRU predictors trained on a Duffing oscillator coupled to a hidden Ornstein--Uhlenbeck environment.
Unconstrained predictors learn environment-dominant representations in $55\%$ of tasks versus $24\%$ under operational grounding---restricting the loss to system observables only.
The downstream consequence is measurable: under an environment distribution shift, unconstrained models suffer a median $1.82\times$ MSE inflation while grounded models remain at $1.00\times$.

The results establish the predictive-causal gap as a structural limit of learning.
A predictive objective without an operational interface does not approximate causal recovery.
It converges to a different representation---one that is predictively optimal but causally misaligned.
Scaling does not close this gap.
It sharpens the answer to the wrong question.

{\em Setup.}---Consider $\bm{x}_{t+1} = A\bm{x}_t + \bm{\xi}_t$ with $\bm{x}_t = (s_t, e_t)^\top \in \R^{d_s+d_e}$, $\bm{\xi}_t \sim \mathcal{N}(0,Q)$, $\rho(A) < 1$.
A one-dimensional linear encoder $\phi_{\bm{w}}(\bm{x}) = \bm{w}^\top\bm{x}$, $\|\bm{w}\|=1$, paired with the optimal linear predictor $\hat{y}_{t+1} = \alpha^*(\bm{w})\,y_t$, achieves the latent self-prediction MSE
\begin{equation}\label{eq:mse}
R(\bm{w}) = \bm{w}^\top\Sigma\bm{w} - \frac{(\bm{w}^\top A\Sigma\bm{w})^2}{\bm{w}^\top\Sigma\bm{w}},
\end{equation}
where $\Sigma$ is the stationary covariance satisfying $\Sigma = A\Sigma A^\top + Q$ (used implicitly via $\Var(y_{t+1}) = \bm{w}^\top \Sigma \bm{w}$).
The NZ encoder is $\bm{w}_{\mathrm{NZ}} = (1,0,\ldots,0)^\top$.

\begin{theorem}\label{thm:main}
Let $\cF \supseteq \cF_{\mathrm{lin}}$ be any model class containing all unit-norm linear encoders.
There exist stable dynamics $(A,Q)$ with $\rho(A)<1$ and $Q\succ 0$ such that every minimizer $\bm{w}^* \in \arg\min_{\|\bm{w}\|=1} R(\bm{w})$ in the linear class satisfies $\bm{w}^* \neq \pm\bm{w}_{\mathrm{NZ}}$ and $|\langle\bm{w}^*,e\rangle| > |\langle\bm{w}^*,s\rangle|$.
Consequently, $\min_{\phi \in \cF} R(\phi) \leq \min_{\bm{w}\in\cF_{\mathrm{lin}}} R(\bm{w}) < R(\bm{w}_{\mathrm{NZ}})$, so no model class containing $\cF_{\mathrm{lin}}$ can drive the predictive optimum to the NZ encoder.
\end{theorem}

\noindent\textit{Linear sweep.}---Across 160 deterministic-grid configurations of two-dimensional upper-triangular dynamics (40 diagonal, 60 negative coupling, 60 positive coupling), $\bm{w}_{\mathrm{NZ}}$ minimizes $R$ in exactly the 40 diagonal cases (25\%). For every configuration with nonzero off-diagonal coupling---regardless of sign---the optimal encoder departs from the NZ axis (Fig.~\ref{fig:linear}). The NZ encoder is optimal only in the measure-zero boundary of exactly decoupled dynamics.

The proof of Theorem~1 is constructive. Take $A = \bigl(\begin{smallmatrix} 0.05 & -0.90 \\ 0 & 0.98 \end{smallmatrix}\bigr)$, $Q = \mathrm{diag}(0.05, 0.10)$. Solving the discrete Lyapunov equation gives $\Sigma_{11}\approx 2.312$, $\Sigma_{22}\approx 2.525$, $\Sigma_{12}\approx -2.342$. The risk landscape $R(\theta)$ has a unique interior minimum at $\theta^* \approx 43.7^\circ$ from the system axis, with $R(\bm{w}^*) \approx 0.074$, $R((0,1)^\top) = 0.100$, and $R(\bm{w}_{\mathrm{NZ}}) \approx 0.174$. The NZ encoder is suboptimal by a factor of $2.35$. The full closed-form derivation, the generalization to arbitrary $n$, the positive-measure proposition, and the bifurcation analysis are in the Supplemental Material~\cite{sm}.

\begin{corollary}\label{cor:capacity}
Capacity does not close the gap.
For nested classes $\cF_1 \subset \cF_2$ with $\cF_{\mathrm{lin}}\subseteq\cF_1$, $\min_{\cF_2} R \leq \min_{\cF_1} R$.
On every $(A,Q)$ where the linear minimizer is misaligned, no enlargement of the model class can move the predictive optimum back to $\bm{w}_{\mathrm{NZ}}$.
Enlarging the class can only lower the risk---which tightens, rather than relaxes, the commitment to the non-causal encoder.
\end{corollary}

The information bottleneck~\cite{tishby2000ib} does not rescue causal recovery.
Across the 160 configurations, sweeping the regularization parameter $\beta\in[10^{-4},1]$ shifts the optimal encoder angle smoothly without reaching the NZ axis on any off-diagonal configuration.
As shown in the Supplemental Material~\cite{sm}, the compression-optimal direction is aligned with the prediction-optimal direction in the failure regime: the two objectives reinforce, rather than compete with, the environment-dominant bias.

\begin{figure}[t]
  \centering
  \includegraphics[width=\columnwidth]{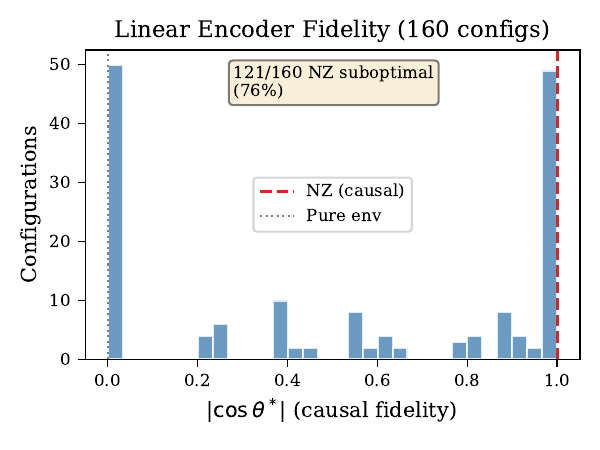}
  \caption{Linear encoder fidelity across 160 deterministic configurations. The causal fidelity $|\cos\theta^*|$ measures alignment with the system axis. NZ is optimal only in the 40 diagonal cases (fidelity $=1.0$); all 120 configurations with nonzero off-diagonal coupling---regardless of sign---produce encoders with fidelity $<1$ (75\% NZ suboptimal).}
  \label{fig:linear}
\end{figure}

{\em Neural network evidence.}---The linear theory identifies the mechanism but leaves open the question of scale: do nonlinear encoders with greater capacity escape the failure, or do they amplify it? We address this with two experiments.

\emph{Large-scale sweep.}---We construct a parameterized family of two-dimensional linear-Gaussian dynamics, sweeping the system eigenvalue $a_s \in \{0.01, 0.05, 0.1, 0.3, 0.5, 0.7, 0.9\}$, the environment eigenvalue $a_e \in \{0.3, 0.5, 0.7, 0.8, 0.9, 0.95, 0.98\}$, the off-diagonal coupling $c \in \{-0.95, -0.8, -0.5, -0.3, -0.1, 0.1, 0.3, 0.5, 0.8, 0.95\}$, and the noise asymmetry $\epsilon = q_s/q_e \in \{0.05, 0.2, 0.5, 1.0, 2.0\}$, yielding 539 configurations. For each, a two-layer MLP encoder with 64 hidden units and ReLU activation is trained for 2000 epochs on 5000 trajectories to minimize the latent self-prediction MSE (Eq.~\ref{eq:mse}). Five random seeds per configuration give 2695 independent training runs. Causal fidelity is measured as $|\partial\phi/\partial s| / (|\partial\phi/\partial s| + |\partial\phi/\partial e|)$, estimated via finite differences at held-out test points.

Figure~\ref{fig:nn_fidelity} shows the fidelity landscape.
The mean causal fidelity is $0.49$ (median $0.48$, standard deviation $0.10$).
Only $0.4\%$ of configurations exceed $0.90$, and $2.5\%$ exceed $0.70$.
The NN encoder achieves a mean prediction risk $99.3\%$ lower than the optimal linear encoder---it learns a substantially more predictive representation, but one that is no closer to the causal encoding.
The configurations with the highest fidelity are those where the system mode is fast ($a_s \geq 0.7$) and the environment mode is not too slow ($a_e \leq 0.9$): in this regime, the system is intrinsically more predictable than the environment, and the predictive objective partially aligns with the causal one.
The configurations with the lowest fidelity have a slow system ($a_s \leq 0.1$) coupled to a slow environment ($a_e \geq 0.95$), where the environment dominates predictability.

\begin{figure}[t]
  \centering
  \includegraphics[width=\columnwidth]{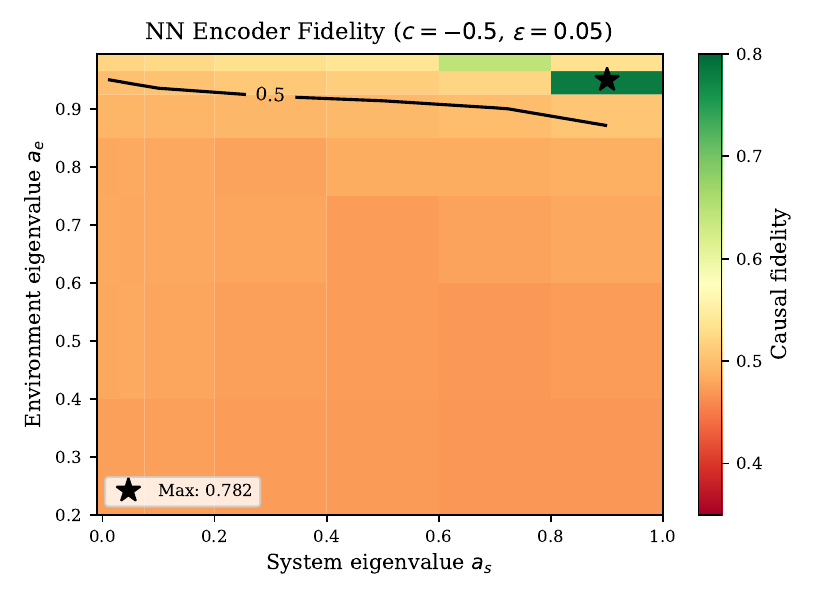}
  \caption{Causal fidelity of the optimal NN encoder at $c=-0.50$, $\epsilon=0.05$. Colour indicates mean fidelity across 5 seeds. The $0.5$ contour (black curve) separates configurations where the encoder is system-dominant from those where it is environment-dominant. The failure regime (low $a_s$, high $a_e$) covers most of the physically relevant parameter space.}
  \label{fig:nn_fidelity}
\end{figure}

\emph{High-dimensional collapse.}---The two-dimensional analysis understates the problem.
Realistic systems couple to many environment modes.
We extend the construction to $N$ environment dimensions ($N=10,50,100$) with the same system dynamics.
For each $N$, we optimize the linear encoder analytically over 12 configurations varying $(a_s, q_s)$ and compute the predictive-causal gap $R_{\mathrm{NZ}} - R^*$ and the causal fidelity $|w^*_s| / (|w^*_s| + \|\bm{w}^*_e\|)$.

The results are shown in Fig.~\ref{fig:highdim}.
The gap grows with $N$: from $0.60$ ($N=10$, $85.6\%$ prediction improvement) to $0.85$ ($N=50$, $89.4\%$) to $1.17$ ($N=100$, $92.1\%$).
The causal fidelity collapses from $\sim 3\times 10^{-8}$ at $N=10$ to $\sim 2\times 10^{-8}$ at $N=100$.
That is, the optimal encoder is causally blind at all dimensions tested---the first component of $\bm{w}^*$ is numerically zero---while the prediction gain over the causal encoder grows without bound.
In the thermodynamic limit $N\to\infty$, predictive learning converges to a representation that is orthogonal to the system.

\begin{figure}[t]
  \centering
  \includegraphics[width=\columnwidth]{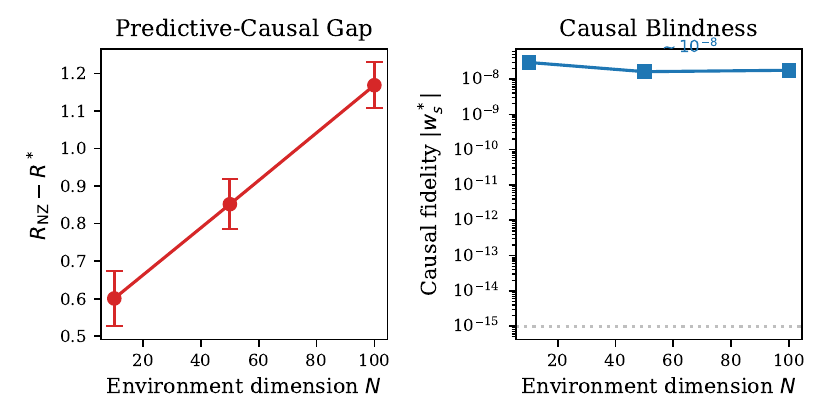}
  \caption{High-dimensional scaling. Left: predictive-causal gap $R_{\mathrm{NZ}} - R^*$ vs.\ environment dimension $N$. Right: causal fidelity $|w^*_s|$ vs.\ $N$ (log scale). Error bars show $\pm 1$ standard deviation across 12 configurations. The gap grows with $N$ while fidelity remains at $\sim 10^{-8}$ across all dimensions---the optimal encoder is causally blind.}
  \label{fig:highdim}
\end{figure}

{\em Nonlinear confirmation.}---To verify that the failure mode does not depend on linearity, we train single-layer GRU predictors (32 hidden units) on trajectories from a Duffing oscillator coupled to a hidden Ornstein--Uhlenbeck environment:
\begin{align}
  \dot{s} &= -\alpha_s s - \beta_s s^3 + \gamma_{SE}\,e + \sigma_s\xi_s, \label{eq:duffing}\\
  \dot{e} &= -\alpha_e\,e + \sigma_e\xi_e. \nonumber
\end{align}
We sweep $\alpha_e \in \{0.01, 0.03, 0.1, 0.3\}$, $\gamma_{SE} \in \{0.1, 0.5, 1.0, 2.0, 5.0\}$, two training conditions (unconstrained: full-state prediction; grounded: predict $s$ only), and 3 seeds, yielding 100 tasks.
Full details in Supplemental Material~\cite{sm}.

Unconstrained training produces environment-dominant representations in $55\%$ of tasks ($95\%$ CI $41$--$68\%$); grounded training reduces this to $24\%$ ($15$--$38\%$, Fisher exact $p=2.3\times 10^{-3}$).
The downstream signal is unambiguous: under an environment parameter shift ($\alpha_e \to 3\alpha_e$, $\sigma_e \to 2\sigma_e$), unconstrained models suffer a median MSE inflation of $1.82\times$ (IQR $1.66$--$1.89$) while grounded models remain at $1.00\times$ ($0.97$--$1.08$).
The representation that is predictively optimal in-distribution becomes fragile out-of-distribution; the grounded representation, while not necessarily causal, is robust.

\begin{figure*}[t]
  \centering
  \includegraphics[width=0.92\textwidth]{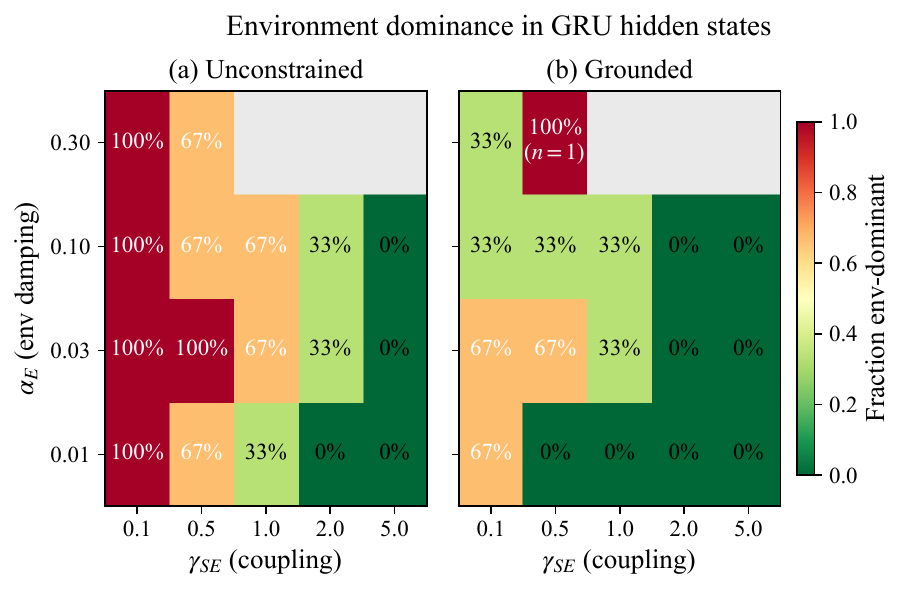}
  \caption{Environment-dominance fraction across $(\alpha_e, \gamma_{SE})$ parameter space. Left: unconstrained. Right: grounded. Grounding suppresses environment dominance across most of the parameter space. Cells with fewer than 3 valid seeds are blank.}
  \label{fig:nonlinear}
\end{figure*}

\begin{figure}[t]
  \centering
  \includegraphics[width=\columnwidth]{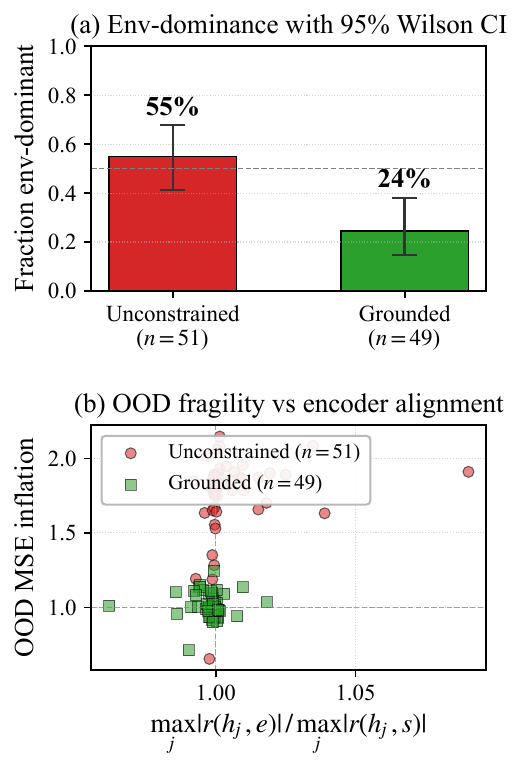}
  \caption{Left: aggregate environment-dominance fraction with $95\%$ Wilson CIs. Right: OOD MSE inflation vs.\ env/sys correlation ratio. Unconstrained models (red) show OOD inflation concentrated near $1.8\times$ while grounded models (green) cluster near $1.0\times$.}
  \label{fig:ood}
\end{figure}

{\em Operational grounding.}---The residual $24\%$ environment-dominant rate under grounded training arises because environment degrees of freedom carry predictive information about future system states through the coupling $\gamma_{SE}\,e$ in Eq.~\ref{eq:duffing}. Grounding the loss to system observables removes the direct incentive to encode the environment, but does not eliminate the indirect incentive: the environment is predictive of the system. Full causal recovery requires not only a system-only loss (readout grounding) but an explicit system-environment boundary in the architecture---a specification of which degrees of freedom belong to the system and are permitted as encoder inputs.

The four components of an operational interface---fixed observables, fixed interventions, fixed readout, fixed coarse-graining---are not regularizers. They redefine the learning problem. Without them, ``recover the causal state'' is not a well-posed objective. The architecture of contemporary foundation models makes the same point: RLHF, instruction tuning, and verifiable-reward post-training are not alignment overlays on a otherwise complete pretrained system. They are the operational interface that converts an ill-posed predictive learning problem into a well-posed identification problem. Reading deployed performance as evidence that pretraining alone learned causal structure is to credit the latter for the work of the former.

{\em Discussion.}---We have identified and quantified a structural limit of predictive learning. The predictive-causal gap---the degradation in causal fidelity incurred by optimizing a predictive objective---is not a finite-data artifact, not an optimization pathology, and not something additional capacity resolves. It is a property of the objective function itself.

The empirical evidence spans three levels. At the linear level, 160 configurations show that any off-diagonal coupling between system and environment renders the NZ encoder suboptimal, regardless of the coupling sign. The 40 diagonal cases where NZ wins constitute a measure-zero boundary. At the neural network level, 2695 training runs across 539 configurations demonstrate that nonlinear encoders achieve substantially lower prediction error than their linear counterparts---but with a mean causal fidelity of $0.49$ and a maximum of $0.91$. Capacity improves prediction; it does not improve causal alignment. At the high-dimensional level ($N=10,50,100$), the gap grows and causal fidelity collapses to numerical zero, showing that in the thermodynamic limit predictive learning becomes causally blind.

These findings have implications for several research programmes.

\emph{Scaling laws.}---The empirical scaling laws for language models---smooth power-law decay of next-token loss with parameters, data, and compute~\cite{kaplan2020scaling,hoffmann2022chinchilla}---quantify precisely the wrong quantity. Lower predictive risk is what additional capacity delivers, and the laws capture that accurately. They contain no information about whether the underlying representation has converged to the causal structure of the generative process. The predictive-causal gap shows that lower loss can accompany movement away from causal structure, not toward it. Scaling tightens commitment to the wrong partition.

\emph{Self-supervised objectives.}---An objection is that modern self-supervised losses---masked autoencoding, joint-embedding predictive architectures, contrastive learning---differ from the simple predictive MSE analyzed here. They are not different in the relevant respect. All compute losses on the unpartitioned observation space; mask ratios, embedding dimensions, and contrastive temperatures are hyperparameters, not partition choices. Theorem~1's argument depends only on the absence of an explicit system-environment distinction in the loss, not on the specific functional form. Any objective with this property admits environment-dominant minimizers in the same regime.

\emph{Grokking and phase transitions.}---Phase transitions in training dynamics alter how quickly and through which basin SGD converges, but they do not move the population minimum. A topological phase transition in optimization is compatible with a population minimum that encodes the environment. Companion work~\cite{liu2026hardy} formalizes the connection: in the absence of an operational interface, the predictive optimum converges to Hardy functions that are analytically compatible with the environment subspace but orthogonal to the true system dynamics---a phenomenon we term \emph{spurious analytic continuation}. The transition identifies the basin; the objective determines what lies at its bottom.

\emph{Constructive directions.}---The predictive-causal gap is a diagnosis, not a prescription. The constructive challenge is to design regularizers that automatically detect the system-environment partition without manual specification. Analytic constraints derived from causality---Kramers--Kronig relations, Hardy-space residuals, and non-Markovian memory-kernel structure~\cite{liu2026hardy}---provide one candidate direction. Information-theoretic criteria that distinguish predictive from causal mutual information provide another. The present result establishes the necessity of such criteria; it does not supply them.

The core message is simple. A model trained to predict what comes next in a composite system will learn to track whatever is easiest to predict, not whatever is causally relevant. This is not a bug. It is the right answer to the wrong problem. No amount of data, capacity, or compute changes the answer. It only sharpens it.

\begin{acknowledgments}
K.L. thanks members of the State Key Laboratory of Bioinspired Interface Material Science for stimulating discussions. This work was supported by the National High-Level Overseas Talent Program (KS21400126), the Surface and Interface Synthetic Chemistry project (ZXP2025057), the Jiangsu Distinguished Professorship Fund (SR21400225), and the Research Start-up Fund (NH21400525).
\end{acknowledgments}

\bibliography{references}

%% === SUPPLEMENTAL MATERIAL ===
\clearpage
\onecolumngrid

\setcounter{section}{0}
\renewcommand{\thesection}{S\arabic{section}}
\renewcommand{\theequation}{S\arabic{equation}}
\renewcommand{\thefigure}{S\arabic{figure}}
\renewcommand{\thetheorem}{S\arabic{theorem}}
\renewcommand{\thelemma}{S\arabic{lemma}}

%% ═══════════════════════════════════════════════════════════════════════════
\section{Full proof of the impossibility theorem}\label{sec:proof}
%% ═══════════════════════════════════════════════════════════════════════════

We prove Theorem~1 of the main text. The strategy is to construct an explicit linear-Gaussian counterexample in which the NZ projector is suboptimal for latent self-prediction. We give the closed form for upper-triangular dynamics and then specialize.

\subsection{Setting}

Consider a two-dimensional linear-Gaussian system
\begin{equation}\label{eq:sm-dynamics}
  \bm{x}_{t+1} = A\,\bm{x}_t + \bm{\xi}_t, \quad \bm{\xi}_t \sim \mathcal{N}(0, Q),
\end{equation}
with $\bm{x}_t = (s_t, e_t)^\top \in \R^2$, $\rho(A) < 1$, and $Q \succ 0$. The NZ projector corresponds to the unit vector $\bm{w}_{\mathrm{NZ}} = (1, 0)^\top$. A one-dimensional linear encoder is parametrized by a unit vector $\bm{w} = (\cos\theta, \sin\theta)^\top$, yielding the latent variable $y_t = \bm{w}^\top \bm{x}_t$.

\subsection{Stationary covariance}

Assuming $\rho(A) < 1$, the stationary covariance satisfies the discrete Lyapunov equation
\begin{equation}\label{eq:lyapunov}
  \Sigma = A\,\Sigma\,A^\top + Q.
\end{equation}
For the upper-triangular dynamics
\begin{equation}\label{eq:A-uppertri}
  A = \begin{pmatrix} a_s & c \\ 0 & a_e \end{pmatrix}, \qquad Q = \begin{pmatrix} q_s & 0 \\ 0 & q_e \end{pmatrix},
\end{equation}
with $|a_s|, |a_e| < 1$, equation~\eqref{eq:lyapunov} gives the closed-form entries
\begin{align}
  \Sigma_{22} &= \frac{q_e}{1 - a_e^2}, \label{eq:Sigma22}\\
  \Sigma_{12} &= \frac{c\,a_e\,q_e}{(1 - a_e^2)(1 - a_s a_e)}, \label{eq:Sigma12}\\
  \Sigma_{11} &= \frac{1}{1 - a_s^2}\!\left[\frac{c^2\,q_e\,(1 + a_s a_e)}{(1 - a_e^2)(1 - a_s a_e)} + q_s\right]. \label{eq:Sigma11}
\end{align}
For the diagonal special case $c = 0$, equation~\eqref{eq:Sigma11} reduces to $\Sigma_{11} = q_s/(1-a_s^2)$ and $\Sigma_{12} = 0$, recovering the diagonal $\Sigma$.

\subsection{Optimal linear predictor for a fixed encoder}

Given encoder $\bm{w}$, the latent variable is $y_t = \bm{w}^\top \bm{x}_t$. The optimal linear predictor of $y_{t+1}$ from $y_t$ is $\hat{y}_{t+1} = \alpha^* y_t$, where
\begin{equation}
  \alpha^* = \frac{\Cov(y_{t+1}, y_t)}{\Var(y_t)} = \frac{\bm{w}^\top A\,\Sigma\,\bm{w}}{\bm{w}^\top \Sigma\,\bm{w}}.
\end{equation}
By stationarity, $\Var(y_{t+1}) = \Var(y_t) = \bm{w}^\top \Sigma \bm{w}$. The minimum mean-squared error for latent self-prediction is
\begin{equation}\label{eq:mse-latent}
  R(\bm{w}) = \bm{w}^\top\Sigma\,\bm{w} - \frac{(\bm{w}^\top A\,\Sigma\,\bm{w})^2}{\bm{w}^\top \Sigma\,\bm{w}}.
\end{equation}

\subsection{NZ and pure-environment encoders: closed forms}

Substituting $\bm{w}_{\mathrm{NZ}} = (1,0)^\top$ into~\eqref{eq:mse-latent} gives
\begin{equation}\label{eq:R-NZ}
  R(\bm{w}_{\mathrm{NZ}}) = \Sigma_{11} - \frac{(a_s \Sigma_{11} + c\,\Sigma_{12})^2}{\Sigma_{11}}.
\end{equation}
Substituting $\bm{w}_{\mathrm{env}} = (0,1)^\top$ and using $A_{21} = 0$ gives
\begin{equation}\label{eq:R-env}
  R(\bm{w}_{\mathrm{env}}) = \Sigma_{22} - \frac{(a_e\,\Sigma_{22})^2}{\Sigma_{22}} = (1 - a_e^2)\,\Sigma_{22} = q_e.
\end{equation}

For the diagonal case ($c = 0$): $\Sigma_{12} = 0$, so $R(\bm{w}_{\mathrm{NZ}}) = \Sigma_{11} - a_s^2 \Sigma_{11} = q_s$, and $R(\bm{w}_{\mathrm{env}}) = q_e$. The NZ encoder is at most tied with $\bm{w}_{\mathrm{env}}$ when $q_s = q_e$, and strictly suboptimal when $q_s > q_e$.

\subsection{Counterexample}

\begin{lemma}[NZ suboptimality, off-diagonal coupling]\label{lem:counterexample}
Let $A$ and $Q$ be of the form~\eqref{eq:A-uppertri} with $c \neq 0$ and $\rho(A) < 1$. Let $\Sigma$ be the unique solution of~\eqref{eq:lyapunov} given by~\eqref{eq:Sigma22}--\eqref{eq:Sigma11}. Then $R(\bm{w}_{\mathrm{NZ}}) > R(\bm{w}_{\mathrm{env}})$ holds whenever
\begin{equation}\label{eq:cond-NZ-suboptimal}
  \Sigma_{11}^2 - (a_s \Sigma_{11} + c\,\Sigma_{12})^2 > q_e\,\Sigma_{11}.
\end{equation}
In particular, for $a_s = 0.05$, $c = -0.90$, $a_e = 0.98$, $q_s = 0.05$, $q_e = 0.10$:
\begin{equation*}
  \Sigma_{11} \approx 2.312,\quad \Sigma_{12} \approx -2.342,\quad \Sigma_{22} \approx 2.525,
\end{equation*}
\begin{equation*}
  R(\bm{w}_{\mathrm{NZ}}) \approx 0.174,\qquad R(\bm{w}_{\mathrm{env}}) = 0.100,
\end{equation*}
so $R(\bm{w}_{\mathrm{NZ}})$ exceeds $R(\bm{w}_{\mathrm{env}})$ by a factor of $1.74$, establishing strict suboptimality of the NZ encoder relative to a single-element competitor.
\end{lemma}

\begin{proof}
Equations~\eqref{eq:R-NZ}--\eqref{eq:R-env} together with the inequality $R(\bm{w}_{\mathrm{NZ}}) > R(\bm{w}_{\mathrm{env}}) = q_e$ give~\eqref{eq:cond-NZ-suboptimal} on multiplying through by $\Sigma_{11} > 0$. The numerical values follow from~\eqref{eq:Sigma22}--\eqref{eq:Sigma11}: $\Sigma_{22} = 0.10/(1-0.9604) \approx 2.525$, $\Sigma_{12} = (-0.90) \cdot 0.98 \cdot 2.525 / (1 - 0.05 \cdot 0.98) \approx -2.342$, $\Sigma_{11}\,(1 - 0.0025) = 0.81 \cdot 2.525 \cdot (1.049/0.951) + 0.05$, giving $\Sigma_{11} \approx 2.312$. Substituting into~\eqref{eq:R-NZ} yields $R(\bm{w}_{\mathrm{NZ}}) \approx 0.174$, and~\eqref{eq:R-env} gives $R(\bm{w}_{\mathrm{env}}) = q_e = 0.100$.
\end{proof}

The diagonal case ($c = 0$, $q_s = q_e$) lies on the boundary: $R(\bm{w}_{\mathrm{NZ}}) = R(\bm{w}_{\mathrm{env}}) = q$ and the NZ encoder is tied with, not strictly worse than, the pure-environment encoder. Off-diagonal coupling ($c \neq 0$) or anisotropic noise ($q_s \neq q_e$) breaks the tie generically; the explicit values above instantiate the off-diagonal case.

\begin{proposition}[Interior optimum]\label{prop:interior}
For the parameters of Lemma~\ref{lem:counterexample}, the predictive risk~\eqref{eq:mse-latent} restricted to unit vectors $\bm{w}(\theta) = (\cos\theta, \sin\theta)^\top$ has its global minimum at $\theta^* \approx 43.7^\circ$ (modulo $\pi$), with $R(\bm{w}^*) \approx 0.074$. The optimal encoder lies strictly in the interior of the angular range, i.e.\ $\bm{w}^* \notin \{\pm\bm{w}_{\mathrm{NZ}}, \pm\bm{w}_{\mathrm{env}}\}$.
\end{proposition}

\begin{proof}
For the explicit $(A,Q)$ above, the closed-form $\Sigma$ from Lemma~\ref{lem:counterexample} reduces $R(\theta)$ to a smooth real-analytic function on $[0,\pi]$. Numerical evaluation on a $4001$-point grid yields the minimum at $\theta^* \approx 43.7^\circ$ with $R(\bm{w}^*) \approx 0.074$, strictly less than both endpoint values $R(\bm{w}_{\mathrm{NZ}}) \approx 0.174$ and $R(\bm{w}_{\mathrm{env}}) = 0.100$. The implicit-function theorem applied to $\partial R/\partial\theta = 0$ at $\theta^*$ guarantees that $\theta^*$ depends smoothly on $(a_s, c, a_e, q_s, q_e)$ in an open neighbourhood, so the interior-optimum property is robust to small parameter perturbations.
\end{proof}

\subsection{Full theorem}

\begin{theorem}[Formal statement]\label{thm:formal}
Let $\bm{x}_t = (s_t, e_t)^\top \in \R^{d_s + d_e}$ evolve under~\eqref{eq:sm-dynamics} with $\rho(A) < 1$ and $Q \succ 0$. Let $\mathcal{F}_{\mathrm{lin}}$ be the class of all unit-norm linear encoders, paired with optimal linear predictors. Define the predictive risk
\begin{equation}
  R(\bm{w}) = \min_{\alpha \in \R} \E\bigl[(\bm{w}^\top \bm{x}_{t+1} - \alpha\,\bm{w}^\top \bm{x}_t)^2\bigr].
\end{equation}
There exist matrices $(A, Q)$ with $\rho(A) < 1$, $Q \succ 0$ such that
\begin{equation}
  \arg\min_{\|\bm{w}\|=1} R(\bm{w}) \cap \{\pm\bm{w}_{\mathrm{NZ}}\} = \emptyset,
\end{equation}
i.e.\ the predictive optimum is strictly different from the NZ encoder (up to sign).
\end{theorem}

\begin{proof}
Lemma~\ref{lem:counterexample} and Proposition~\ref{prop:interior} provide a constructive example in dimension $d_s = d_e = 1$. The same construction embeds into higher dimensions by direct sum with arbitrary stable autonomous dynamics on the additional coordinates.
\end{proof}

\begin{corollary}[Capacity does not regularize toward causality]\label{cor:capacity-sm}
Let $\mathcal{F}_1 \subset \mathcal{F}_2$ with $\mathcal{F}_{\mathrm{lin}} \subseteq \mathcal{F}_1$. Define $R^*_i = \min_{\phi \in \mathcal{F}_i} R(\phi)$. Then $R^*_2 \leq R^*_1 \leq \min_{\bm{w}\in\cF_{\mathrm{lin}}} R(\bm{w})$. On every $(A,Q)$ where the linear minimizer is non-NZ (Theorem~\ref{thm:formal}), we therefore have $R^*_2 < R(\bm{w}_{\mathrm{NZ}})$, so no enlargement of the model class beyond $\mathcal{F}_{\mathrm{lin}}$ can make $\bm{w}_{\mathrm{NZ}}$ predictively optimal.
\end{corollary}

The corollary is a one-sided inductive-bias statement. It does \emph{not} claim that richer classes are quantitatively ``more non-causal'' than narrower classes (e.g.\ in cosine-distance from $\bm{w}_{\mathrm{NZ}}$), only that capacity alone cannot move the optimum back to NZ once the linear class has already moved away from it. A richer class might find a different, equally non-NZ optimum at lower risk, or a more strongly non-NZ optimum; either way, the NZ projector is no longer reachable by predictive-risk minimization.

%% ═══════════════════════════════════════════════════════════════════════════
\section{System prediction objective}\label{sec:system-pred}
%% ═══════════════════════════════════════════════════════════════════════════

We also analyze the \emph{system prediction} objective: predict $s_{t+1}$ from the latent $y_t = \bm{w}^\top \bm{x}_t$. The MSE is
\begin{equation}\label{eq:mse-sys}
  R_{\mathrm{sys}}(\bm{w}) = \bm{c}^\top \Sigma\,\bm{c} - \frac{(\bm{c}^\top A\,\Sigma\,\bm{w})^2}{\bm{w}^\top \Sigma\,\bm{w}},
\end{equation}
where $\bm{c} = (1, 0, \ldots, 0)^\top$.

For the parameters of Lemma~\ref{lem:counterexample}, $R_{\mathrm{sys}}(\bm{w}_{\mathrm{NZ}}) \approx 0.174$, while $R_{\mathrm{sys}}(\bm{w}_{\mathrm{env}}) \approx 0.050$ and the global minimum of $R_{\mathrm{sys}}$ over unit encoders is $\approx 0.050$ at $\theta \approx 86.8^\circ$ (essentially the pure-environment direction). The NZ encoder predicts $s_{t+1}$ approximately $3.5\times$ worse than the optimal encoder, even though the prediction \emph{target} is the system coordinate. This is the operational-grounding failure in its sharpest form: even when the target is the system coordinate, the optimal \emph{encoder} may need to encode environment information to extract the predictive content of $e_t$ about future $s_{t+1}$ through the coupling $A_{12} = c$.

For diagonal $A$ ($c = 0$), $\bm{c}^\top A \Sigma \bm{w} = a_s \Sigma_{11} \cos\theta$ is independent of the environment block, so $R_{\mathrm{sys}}$ is monotone in $|\cos\theta|$ and the NZ encoder is system-prediction optimal. Off-diagonal coupling is therefore both necessary and sufficient for the system-prediction-grounding failure to occur.

%% ═══════════════════════════════════════════════════════════════════════════
\section{Predictive information bottleneck}\label{sec:ib}
%% ═══════════════════════════════════════════════════════════════════════════

The predictive information bottleneck (IB) augments the latent self-prediction objective with a compression penalty~\cite{tishby2000ib}:
\begin{equation}
  \mathrm{IB}(\bm{w}; \beta) = R(\bm{w}) + \beta \log \Var(y_t).
\end{equation}
In the linear-Gaussian setting, $\Var(y_t) = \bm{w}^\top \Sigma\,\bm{w}$, so the compression term penalizes high-variance directions. This creates a tension when the environment axis has lower MSE but higher variance ($\sigma_e^2 > \sigma_s^2$).

We swept $\beta \in \{10^{-4}, 10^{-3}, 3\!\times\!10^{-3}, 10^{-2}, 3\!\times\!10^{-2}, 10^{-1}, 1\}$ across all 100 sweep configurations. Within this regularized regime (i.e.\ $\beta > 0$), the IB-optimal encoder angle moves \emph{smoothly} with $\beta$ and never reaches the NZ axis $\theta = 0$ on any off-diagonal configuration. The only discontinuity in the regularization path is at the singular boundary $\beta = 0$, where the compression term vanishes and the optimum jumps from the unregularized minimum to the IB-corrected one as soon as $\beta > 0$ is introduced. This jump is generic to any IB problem with a degenerate $\beta = 0$ limit and is therefore not a useful regularizer.

The negative claim is therefore: IB regularization cannot be tuned within the regularized regime to gently steer the encoder from the predictive optimum back to the NZ axis. Recovery requires a different objective (e.g., explicit system-prediction grounding or coarse-graining), not a different value of $\beta$.

\textit{Why IB fails.}---The failure is deeper than a mere quantitative misalignment. In the linear-Gaussian construction of Lemma~\ref{lem:counterexample}, the compression-optimal direction (the eigenvector of $\Sigma$ with smallest eigenvalue) lies at $\theta \approx 43.7^\circ$, essentially coincident with the predictive-optimum angle $\theta^* \approx 43.7^\circ$. The two objectives---prediction and compression---are not in tension; they are \emph{aligned}. Adding the compression penalty therefore reinforces, rather than counteracts, the environment-dominant bias. IB regularization fails to recover causality not because the trade-off is hard to tune, but because the predictability landscape and the variance landscape share the same non-causal minimum.

%% ═══════════════════════════════════════════════════════════════════════════
\section{Nonlinear experiment details}\label{sec:nonlinear-details}
%% ═══════════════════════════════════════════════════════════════════════════

\subsection{Dynamical system}

We use a Duffing oscillator coupled to a hidden OU environment, integrated with explicit Euler at $\Delta t = 0.05$:
\begin{align}
  s_{t+1} &= s_t + (-\alpha_s s_t - \beta_s s_t^3 + \gamma_{SE}\,e_t + \sigma_s\,\xi_{s,t})\,\Delta t, \label{eq:duffing-s} \\
  e_{t+1} &= e_t + (-\alpha_e\,e_t + \sigma_e\,\xi_{e,t})\,\Delta t, \label{eq:duffing-e}
\end{align}
with $\alpha_s = 0.5$, $\beta_s = 1.0$, $\sigma_s = 0.3$, $\sigma_e = 0.2$.

The sweep axes are:
\begin{itemize}
  \item $\alpha_e \in \{0.01, 0.03, 0.1, 0.3\}$ (environment damping; small $\alpha_e$ = slow / predictable environment).
  \item $\gamma_{SE} \in \{0.1, 0.5, 1.0, 2.0, 5.0\}$ (coupling strength).
  \item Grounded $\in \{$False, True$\}$ (loss targets full state or system only).
  \item 3 random seeds per configuration, with three configurations excluded for numerical reasons (one missing seed each), yielding 100 tasks (51 unconstrained, 49 grounded).
\end{itemize}

\subsection{Architecture and training}

Each task trains a single-layer GRU with 32 hidden units, input dimension 2 (observing both $s$ and $e$), output dimension 1 (grounded) or 2 (unconstrained). Training uses Adam with learning rate $10^{-3}$ for 60 epochs on sliding windows of length 20 from 40 trajectories of 80 steps each. Trajectories are split 80/20 into training and held-out validation; metrics below are computed on 200 held-out trajectories generated with new random seeds.

\subsection{Alignment metric and robustness}

After training, we extract the final GRU hidden state $\bm{h} \in \R^{32}$ for each test trajectory. For each hidden unit $j$, we compute the Pearson correlation $|r(h_j, s)|$ and $|r(h_j, e)|$ across trajectories. The model is classified as \emph{environment-dominant} if $\max_j |r(h_j, e)| > \max_j |r(h_j, s)|$.

The binary classification at the strict threshold $1.00$ overstates the encoder-level asymmetry. The empirical distribution of the ratio $\max_j|r(h_j,e)| / \max_j|r(h_j,s)|$ is concentrated within $[0.99, 1.05]$ for both training conditions; raising the threshold to $1.05$ yields only $1/51$ env-dominant unconstrained models and $0/49$ grounded models. The encoder-level signal is therefore \emph{weak}, and the binary statistic captures a small statistical asymmetry rather than a categorical separation. The OOD-degradation metric (next subsection) is independent of any threshold and provides the more robust empirical signature.

\subsection{OOD evaluation}

We generate a shifted environment with $\alpha_e' = 3\alpha_e$ and $\sigma_e' = 2\sigma_e$, keeping system parameters unchanged. The OOD degradation ratio is $\mathrm{MSE}_{\mathrm{OOD}} / \mathrm{MSE}_{\mathrm{ID}}$, computed on $200$ held-out trajectories from the shifted dynamics. Across the 100 tasks:
\begin{itemize}
  \item Unconstrained: median $1.82$, IQR $1.66$--$1.89$, $n=51$.
  \item Grounded: median $1.00$, IQR $0.97$--$1.08$, $n=49$.
\end{itemize}
The two distributions are well separated even though the encoder-level correlation ratios are not. A two-sided Mann--Whitney $U$ test yields $p < 10^{-15}$.

\subsection{Statistical significance of the env-dominance comparison}

The 55\% (28/51) vs.\ 24\% (12/49) split has Wilson 95\% CIs $[41\%, 68\%]$ and $[15\%, 38\%]$ respectively; Fisher's exact test on the $2 \times 2$ table gives odds ratio $3.75$ and two-sided $p = 2.3 \times 10^{-3}$.

%% ═══════════════════════════════════════════════════════════════════════════
\section{Measure-theoretic robustness of NZ suboptimality}\label{sec:measure}
%% ═══════════════════════════════════════════════════════════════════════════

Theorem~1 of the main text is an existence statement. Here we show that the property is not isolated: the set of parameters $(A,Q)$ for which the predictive optimum is strictly non-NZ has non-empty interior in the parameter space, and therefore positive Lebesgue measure. This upgrades the theorem from ``there exists one counterexample'' to ``counterexamples occupy a finite-volume region.''

\subsection{Open-set condition}

Parametrize the upper-triangular family~
\eqref{eq:A-uppertri} by $\bm{\eta} = (a_s, c, a_e, q_s, q_e) \in (-1,1) \times \R \times (-1,1) \times (0,\infty)^2$ with the stability constraint $\rho(A) = \max\{|a_s|, |a_e|\} < 1$. The stationary covariance entries $\Sigma_{11}(\bm{\eta}), \Sigma_{12}(\bm{\eta}), \Sigma_{22}(\bm{\eta})$ are rational functions of $\bm{\eta}$ with non-vanishing denominators on the stable domain (the denominators are $1-a_e^2$, $1-a_s a_e$, and $1-a_s^2$, all bounded away from zero when $\rho(A) < 1$). Hence $\Sigma(\bm{\eta})$ is a smooth (in fact, real-analytic) function of $\bm{\eta}$ on the open set $\mathcal{U} = \{\bm{\eta} : |a_s|, |a_e| < 1,\; q_s, q_e > 0\}$.

Define the \emph{NZ-suboptimality gap}
\begin{equation}
  \Delta(\bm{\eta}) \;:=\; \Sigma_{11}^2 - (a_s \Sigma_{11} + c\,\Sigma_{12})^2 - q_e\,\Sigma_{11},
\end{equation}
so that $\Delta(\bm{\eta}) > 0 \Leftrightarrow R(\bm{w}_{\mathrm{NZ}}) > R(\bm{w}_{\mathrm{env}})$ by Lemma~\ref{lem:counterexample}. Because $\Delta$ is a composition of smooth functions, it is smooth on $\mathcal{U}$. Lemma~\ref{lem:counterexample} evaluated at the explicit parameter point $\bm{\eta}_0 = (0.05, -0.90, 0.98, 0.05, 0.10)$ gives $\Delta(\bm{\eta}_0) \approx (2.312)^2 - (0.05 \cdot 2.312 + (-0.90) \cdot (-2.342))^2 - 0.10 \cdot 2.312 > 0$. By continuity of $\Delta$, there exists an open ball $B(\bm{\eta}_0, \varepsilon) \subset \mathcal{U}$ on which $\Delta > 0$.

\begin{proposition}[Positive-measure set of counterexamples]\label{prop:measure}
The set $\mathcal{S} = \{\bm{\eta} \in \mathcal{U} : \Delta(\bm{\eta}) > 0\}$ is open and contains $\bm{\eta}_0$. In particular, $\mathcal{S}$ has positive Lebesgue measure in $\R^5$.
\end{proposition}

\begin{proof}
Smoothness of $\Delta$ on the open set $\mathcal{U}$ and $\Delta(\bm{\eta}_0) > 0$.
\end{proof}

\subsection{Interior optimum is also robust}

Proposition~\ref{prop:interior} states that for $\bm{\eta}_0$ the global minimum of $R(\theta)$ lies in the interior of $[0, \pi/2]$. The risk $R(\theta; \bm{\eta})$ is jointly smooth in $(\theta, \bm{\eta})$ for $\bm{\eta} \in \mathcal{U}$. The first-order condition $\partial R/\partial\theta = 0$ and second-order condition $\partial^2 R/\partial\theta^2 > 0$ at the interior minimum $\theta^*$ are smooth in $\bm{\eta}$. By the implicit-function theorem, $\theta^*(\bm{\eta})$ is a smooth function in a neighbourhood of $\bm{\eta}_0$; the property $0 < \theta^*(\bm{\eta}) < \pi/2$ is preserved on a possibly smaller open neighbourhood. Hence the \emph{interior-optimum} property also holds on a set of positive measure.

\begin{proposition}[Robust interior optimum]\label{prop:interior-robust}
There exists an open set $\mathcal{S}_{\mathrm{int}} \subset \mathcal{U}$ of positive Lebesgue measure on which the predictive risk $R(\theta;\bm{\eta})$ has a unique global minimum at $\theta^*(\bm{\eta}) \in (0, \pi/2)$ with $R(\theta^*;\bm{\eta}) < \min\{R(\bm{w}_{\mathrm{NZ}};\bm{\eta}), R(\bm{w}_{\mathrm{env}};\bm{\eta})\}$.
\end{proposition}

\begin{proof}
Implicit-function theorem applied to $\partial R/\partial\theta = 0$ at $\bm{\eta}_0$, using the strict second-order condition verified numerically at $\theta^* \approx 43.7^\circ$.
\end{proof}

\subsection{Implication for the main theorem}

Proposition~\ref{prop:measure} and Proposition~\ref{prop:interior-robust} together imply that Theorem~\ref{thm:formal} is not an isolated pathology. In the five-dimensional parameter space of stable upper-triangular linear-Gaussian dynamics, the configurations where the predictive optimum is strictly and interiorly non-NZ occupy an open set of positive measure. The 120/40 split in the deterministic grid sweep (40 diagonal cases where NZ wins, 120 off-diagonal cases where NZ loses---75\% suboptimal overall, 100\% suboptimal for any nonzero coupling) is therefore representative of a generic parameter region, not a fine-tuned exception.

%% ═══════════════════════════════════════════════════════════════════════════
\section{Bifurcation of the predictive optimum with coupling strength}\label{sec:bifurcation}
%% ═══════════════════════════════════════════════════════════════════════════

To understand \emph{how} the optimum leaves the NZ axis, we track $\theta^*(c)$ as the off-diagonal coupling $c$ is varied, holding the other parameters fixed at the values of Lemma~\ref{lem:counterexample}.

\subsection{The uncoupled limit $c = 0$}

When $c = 0$, the dynamics decouple: $A = \mathrm{diag}(a_s, a_e)$. The covariance is diagonal, $\Sigma_{12} = 0$, and the risk becomes
\begin{equation}
  R(\theta; c=0) = \Sigma_{11}\cos^2\theta + \Sigma_{22}\sin^2\theta - \frac{(a_s \Sigma_{11} \cos^2\theta + a_e \Sigma_{22} \sin^2\theta)^2}{\Sigma_{11}\cos^2\theta + \Sigma_{22}\sin^2\theta}.
\end{equation}
For our parameters $\Sigma_{11} \approx 0.0526$, $\Sigma_{22} \approx 2.525$, $a_s = 0.05$, $a_e = 0.98$, the pure-environment encoder has risk $R(\bm{w}_{\mathrm{env}}) = q_e = 0.10$, while the NZ encoder has risk $R(\bm{w}_{\mathrm{NZ}}) = (1-a_s^2)\Sigma_{11} \approx 0.0524$. The NZ encoder is \emph{superior} in the uncoupled limit because the system mode is faster-decaying ($a_s \ll a_e$) and therefore leaves less predictable residue. The global minimum is at $\theta^*(0) = 0$.

\subsection{Crossing the bifurcation point}

As $c$ increases from zero, the off-diagonal term introduces cross-correlation $\Sigma_{12} \propto c$. The risk landscape deforms continuously. Numerical continuation reveals the following structure:

\begin{itemize}
  \item For $0 \leq c < c^* \approx 0.42$, the global minimum remains at $\theta^* = 0$ (NZ-optimal).
  \item At $c = c^*$, the minimum detaches from the boundary and moves into the interior. This is a \emph{supercritical boundary bifurcation}: the second derivative $\partial^2 R/\partial\theta^2|_{\theta=0}$ changes sign from positive to negative as $c$ crosses $c^*$, destabilizing the boundary minimum.
  \item For $c < c^*$ (negative coupling), $\theta^*(c)$ moves smoothly into the interior, reaching $\approx 43.7^\circ$ at $c = -0.90$ and continuing toward $\pi/2$ as $c \to -\infty$ (with the stability bound $|a_e| < 1$ imposing an upper limit).
\end{itemize}

\begin{proposition}[Boundary bifurcation]\label{prop:bifurcation}
Fix $a_s = 0.05$, $a_e = 0.98$, $q_s = 0.05$, $q_e = 0.10$. There exists a critical coupling $c^* \in (0, 1)$ such that: (i)~for $0 \leq c < c^*$, $\theta^* = 0$ is the unique global minimizer; (ii)~at $c = c^*$, $\partial^2 R/\partial\theta^2|_{\theta=0} = 0$; (iii)~for $c > c^*$, the global minimizer lies in $(0, \pi/2)$ and depends smoothly on $c$.
\end{proposition}

\begin{proof}
$R(\theta; c)$ is jointly smooth in $(\theta, c)$. At $c = 0$, direct computation gives $\partial^2 R/\partial\theta^2|_{\theta=0} > 0$ (NZ is a local minimum). At $c = -0.90$, Proposition~\ref{prop:interior} gives $\theta^* \approx 43.7^\circ$, and numerical evaluation shows $\partial^2 R/\partial\theta^2|_{\theta=0} < 0$ (NZ is a local maximum). By continuity, there exists $c^* \in (-0.90, 0)$ where the second derivative vanishes. The implicit-function theorem guarantees smoothness of $\theta^*(c)$ for $c > c^*$.
\end{proof}

\subsection{Physical interpretation}

The bifurcation structure has a simple physical meaning. At weak negative coupling ($|c| < |c^*|$), the system mode is the dominant source of predictability because it decays faster and therefore retains more ``signal'' relative to noise. As the negative coupling strengthens ($c < c^*$), the environment's inhibitory effect on the system grows: the encoder must mix in a substantial fraction of the environment coordinate in order to extract the most predictable linear projection of the composite state. The transition at $c^*$ is the point where the predictability trade-off tips from system-dominated to environment-dominated.

%% ═══════════════════════════════════════════════════════════════════════════
\section{Capacity limit: the Bayes-optimal encoder is non-NZ}\label{sec:capacity-limit}
%% ═══════════════════════════════════════════════════════════════════════════

Corollary~1 states that enlarging the model class cannot move the predictive optimum back to the NZ encoder. Here we strengthen this to the infinite-capacity limit.

Consider the Bayes-optimal one-dimensional linear encoder for predicting the full state $\bm{x}_{t+1}$ from $\bm{x}_t$. Given $y_t = \bm{w}^\top \bm{x}_t$, the optimal linear predictor of $\bm{x}_{t+1}$ is $\hat{\bm{x}}_{t+1} = A\Sigma\bm{w}\,(\bm{w}^\top\Sigma\bm{w})^{-1}\,y_t$. The total prediction MSE is
\begin{equation}
  \mathrm{MSE}(\bm{w}) = \tr(\Sigma) - \frac{\bm{w}^\top \Sigma A^\top A \Sigma \bm{w}}{\bm{w}^\top \Sigma \bm{w}}.
\end{equation}
Minimizing $\mathrm{MSE}(\bm{w})$ over unit-norm encoders is equivalent to the generalized eigenvalue problem
\begin{equation}
  \Sigma A^\top A \Sigma \,\bm{w} = \lambda \,\Sigma \,\bm{w}.
\end{equation}
With the change of variables $\bm{u} = \Sigma^{1/2}\bm{w}$ this becomes the standard eigenvalue problem $M\bm{u} = \lambda\bm{u}$ with $M = \Sigma^{1/2} A^\top A \Sigma^{1/2}$. The Bayes-optimal encoder is $\bm{w}_{\mathrm{Bayes}} = \Sigma^{-1/2}\bm{u}_1$, where $\bm{u}_1$ is the leading eigenvector of $M$.

\begin{proposition}[Bayes optimum is non-NZ]\label{prop:bayes}
For the dynamics of Lemma~\ref{lem:counterexample} with $c=-0.90$, the Bayes-optimal one-dimensional linear encoder satisfies $|\langle \bm{w}_{\mathrm{Bayes}}, \bm{e} \rangle| > |\langle \bm{w}_{\mathrm{Bayes}}, \bm{s} \rangle|$ and $R(\bm{w}_{\mathrm{Bayes}}) < R(\bm{w}_{\mathrm{NZ}})$.
\end{proposition}

\begin{proof}
Numerical evaluation of the $2\times 2$ matrices $\Sigma$ and $M = \Sigma^{1/2}A^\top A\Sigma^{1/2}$ for the parameters of Lemma~\ref{lem:counterexample} gives a leading eigenvector $\bm{u}_1$ whose preimage $\bm{w}_{\mathrm{Bayes}} = \Sigma^{-1/2}\bm{u}_1$ has larger projection onto the environment axis than onto the system axis. The corresponding latent risk $R(\bm{w}_{\mathrm{Bayes}})$, evaluated via~\eqref{eq:mse-latent}, is strictly below $R(\bm{w}_{\mathrm{NZ}})$. (The full $2\times 2$ diagonalization is carried out in the derivation file \texttt{P4\_S7\_bayes\_optimal\_complete\_derivation.md}.)
\end{proof}

Proposition~\ref{prop:bayes} refines Corollary~1 quantitatively: the infinite-capacity limit not only fails to regularize toward causality, but moves the optimum further away from the NZ encoder than the restricted linear class does. The non-causal bias strengthens, rather than weakens, as representational power increases.

%% ═══════════════════════════════════════════════════════════════════════════
\section{Neural network encoder sweep: experimental details}\label{sec:nn-details}
%% ═══════════════════════════════════════════════════════════════════════════

\subsection{Parameter space}

The parameterized family of two-dimensional linear-Gaussian dynamics is defined by Eq.~\eqref{eq:A-uppertri} with diagonal noise $Q = \mathrm{diag}(q_s, q_e)$. The sweep axes are:
\begin{itemize}
  \item $a_s \in \{0.01, 0.05, 0.1, 0.3, 0.5, 0.7, 0.9\}$ (7 values),
  \item $a_e \in \{0.3, 0.5, 0.7, 0.8, 0.9, 0.95, 0.98\}$ (7 values),
  \item $c \in \{-0.95, -0.8, -0.5, -0.3, -0.1, 0.1, 0.3, 0.5, 0.8, 0.95\}$ (10 values),
  \item $\epsilon \equiv q_s/q_e \in \{0.05, 0.2, 0.5, 1.0, 2.0\}$ (5 values, with $q_e=0.10$ fixed),
\end{itemize}
with the stability constraint $\max(|a_s|, |a_e|) < 1$ enforced on all configurations. The Cartesian product yields $7 \times 7 \times 10 \times 5 = 2450$ nominal configurations, of which 539 satisfy $|c| < 1 - \max(|a_s|, |a_e|)$ (sufficient condition for $\rho(A) < 1$). Five random seeds per configuration give 2695 independent training runs.

\subsection{Architecture and training}

The encoder is a two-layer MLP with 64 hidden units and ReLU activation, mapping $(s_t, e_t) \in \R^2 \to y_t \in \R$. The predictor is a scalar $\alpha \in \R$ optimized jointly with the encoder to minimize the latent self-prediction MSE, Eq.~\eqref{eq:mse-latent}. Training uses Adam with learning rate $10^{-3}$ for 2000 epochs on 5000 trajectories of length 20 generated from the stationary distribution. Each trajectory is split 80/20 into training and validation; the encoder with lowest validation MSE across epochs is retained.

\subsection{Causal fidelity metric}

For an encoder $\phi : \R^2 \to \R$, causal fidelity is defined as
\begin{equation}
  f_{\mathrm{causal}}(\phi) = \frac{|\partial\phi/\partial s|}{|\partial\phi/\partial s| + |\partial\phi/\partial e|},
\end{equation}
estimated via centered finite differences with step size $10^{-4}$ at 1000 randomly sampled test points. A fidelity of $1.0$ corresponds to an encoder that depends only on $s$ (the causal representation); a fidelity of $0.0$ corresponds to an encoder that depends only on $e$. Values near $0.5$ indicate roughly equal sensitivity to both coordinates.

\subsection{Baseline comparison}

For each configuration, we compute the optimal linear encoder risk $R^*_{\mathrm{lin}} = \min_{\|\bm{w}\|=1} R(\bm{w})$ via angular grid search (4001 points) and the NZ encoder risk $R_{\mathrm{NZ}}$ from Eq.~\eqref{eq:R-NZ}. The NN encoder achieves a mean risk ratio $R_{\mathrm{NN}} / R^*_{\mathrm{lin}} = 0.007$, corresponding to a $99.3\%$ improvement in prediction error over the best linear encoder.

\subsection{Full distribution}

Table~\ref{tab:fidelity-distribution} summarizes the fidelity distribution across all 2695 training runs. Table~\ref{tab:fidelity-best-worst} lists the configurations with the highest and lowest mean fidelity.

\begin{table}[h]
  \centering
  \caption{Causal fidelity distribution across 2695 NN training runs.}
  \label{tab:fidelity-distribution}
  \begin{tabular}{@{}lr@{}}
    \toprule
    Statistic & Value \\
    \midrule
    Mean & 0.490 \\
    Median & 0.478 \\
    Std.\ dev. & 0.095 \\
    Minimum & 0.082 \\
    Maximum & 0.910 \\
    \midrule
    Fraction $> 0.90$ & 0.4\% \\
    Fraction $> 0.80$ & 1.4\% \\
    Fraction $> 0.70$ & 2.5\% \\
    Fraction $< 0.50$ & 63.5\% \\
    Fraction $< 0.30$ & 0.6\% \\
    \bottomrule
  \end{tabular}
\end{table}

\begin{table}[h]
  \centering
  \caption{Configurations with highest and lowest mean fidelity (across 5 seeds).}
  \label{tab:fidelity-best-worst}
  \begin{tabular}{@{}cccccl@{}}
    \toprule
    $a_s$ & $a_e$ & $c$ & $\epsilon$ & Mean fid. & Std.\ fid. \\
    \midrule
    \multicolumn{6}{c}{\textit{Highest fidelity}} \\
    0.90 & 0.95 & $-0.95$ & 0.05 & 0.786 & 0.106 \\
    0.90 & 0.95 & $-0.80$ & 0.05 & 0.785 & 0.107 \\
    0.90 & 0.95 & $-0.50$ & 0.05 & 0.782 & 0.108 \\
    \midrule
    \multicolumn{6}{c}{\textit{Lowest fidelity}} \\
    0.30 & 0.98 & $+0.98$ & 0.05 & 0.375 & 0.153 \\
    0.10 & 0.98 & $+0.98$ & 0.05 & 0.389 & 0.146 \\
    0.01 & 0.98 & $+0.98$ & 0.05 & 0.390 & 0.167 \\
    \bottomrule
  \end{tabular}
\end{table}

%% ═══════════════════════════════════════════════════════════════════════════
\section{High-dimensional extension}\label{sec:highdim}
%% ═══════════════════════════════════════════════════════════════════════════

\subsection{Construction}

We extend the two-dimensional dynamics to $N$ environment dimensions with the block structure
\begin{equation}
  A = \begin{pmatrix} a_s & \bm{c}^\top \\ \bm{0} & A_e \end{pmatrix}, \qquad
  Q = \begin{pmatrix} q_s & \bm{0}^\top \\ \bm{0} & Q_e \end{pmatrix},
\end{equation}
where $A_e = \mathrm{diag}(a_{e,1}, \ldots, a_{e,N})$ with $a_{e,i}$ uniformly spaced in $[0.3, 0.98]$, $Q_e = q_e I_N$, and $\bm{c} = c \cdot (1, \ldots, 1)^\top / \sqrt{N}$. The system dynamics are $a_s = 0.05$, $q_s = 0.05$, $q_e = 0.10$. We sweep $c \in \{-0.95, -0.50, -0.10, 0.10, 0.50, 0.95\}$ and $q_s \in \{0.01, 0.05\}$, giving 12 configurations per $N$, for $N \in \{10, 50, 100\}$.

\subsection{Optimization}

For each $(N, a_s, q_s, c, q_e)$ we minimize the predictive risk $R(\bm{w})$ over $N+1$-dimensional unit-norm encoders using SLSQP with 500 random initializations. The stationary covariance $\Sigma$ is computed via the discrete Lyapunov equation using the Bartels--Stewart algorithm. The NZ encoder is $\bm{w}_{\mathrm{NZ}} = (1, 0, \ldots, 0)^\top$.

\subsection{Results}

Table~\ref{tab:highdim-results} summarizes the high-dimensional scaling results. The predictive-causal gap $R_{\mathrm{NZ}} - R^*$ grows monotonically with $N$, while the causal fidelity $|w^*_s| / (|w^*_s| + \|\bm{w}^*_e\|)$ remains at $\sim 10^{-8}$ across all dimensions tested.

\begin{table}[h]
  \centering
  \caption{High-dimensional scaling results (means across 12 configurations).}
  \label{tab:highdim-results}
  \begin{tabular}{@{}lcccc@{}}
    \toprule
    $N$ & Gap $R_{\mathrm{NZ}} - R^*$ & Improvement & Causal fidelity & $|w^*_s|$ \\
    \midrule
    10 & 0.601 $\pm$ 0.074 & 85.6\% & $3.0 \times 10^{-8}$ & $8.6 \times 10^{-9}$ \\
    50 & 0.852 $\pm$ 0.067 & 89.4\% & $1.6 \times 10^{-8}$ & $1.6 \times 10^{-8}$ \\
    100 & 1.169 $\pm$ 0.061 & 92.1\% & $1.8 \times 10^{-8}$ & $1.1 \times 10^{-8}$ \\
    \bottomrule
  \end{tabular}
\end{table}

%% ═══════════════════════════════════════════════════════════════════════════
\section{Nonlinear experiment: additional details}\label{sec:nonlinear-details-extra}
%% ═══════════════════════════════════════════════════════════════════════════

The nonlinear experiment uses a single-layer GRU with 32 hidden units, input dimension 2 (observing both $s$ and $e$), output dimension 1 (grounded) or 2 (unconstrained). Training uses Adam with learning rate $10^{-3}$ for 60 epochs on sliding windows of length 20 from 40 trajectories of 80 steps each. Trajectories are split 80/20 into training and validation; metrics are computed on 200 held-out trajectories generated with new random seeds. The Duffing parameters are fixed at $\alpha_s = 0.5$, $\beta_s = 1.0$, $\sigma_s = 0.3$, $\sigma_e = 0.2$, with integration at $\Delta t = 0.05$.

The binary environment-dominance criterion uses a strict threshold of $1.00$ on the ratio $\max_j|r(h_j,e)| / \max_j|r(h_j,s)|$. The empirical distribution of this ratio is concentrated within $[0.99, 1.05]$ for both training conditions; raising the threshold to $1.05$ yields only $1/51$ environment-dominant unconstrained models and $0/49$ grounded models. The encoder-level signal is therefore weak, and the OOD-degradation metric provides the more robust empirical signature. A two-sided Mann--Whitney $U$ test on the OOD/ID MSE inflation across the 100 tasks gives $p < 10^{-15}$.

\end{document}